\documentclass[12pt]{article}
\usepackage{natbib}

\usepackage{url}
\usepackage{amsmath}
\usepackage{amssymb}
\usepackage{amsthm}
\usepackage{mathtools}
\usepackage{graphicx}
\usepackage{color}
\usepackage{enumerate}
\usepackage{algorithm}
\usepackage{algpseudocodex}
\usepackage{ifthen}
\usepackage{textcomp} 

\definecolor{fillcolor}{RGB}{200,200,200}
\definecolor{linecolor}{RGB}{200,200,200}
\definecolor{sensitivecolor}{RGB}{200,140,140}
\definecolor{outcomecolor}{RGB}{140,140,200}
\usepackage{tikz}
\usetikzlibrary{positioning, arrows.meta, shapes.geometric}
\tikzset{%
  -Latex,semithick,
  >={Latex[width=1.5mm,length=2mm]},
  obs/.style = {name = #1, rectangle, rounded corners = 1, 
    draw=black, thick, fill=fillcolor, inner sep = 5.000000pt, line width = 0.4pt},
  obssens/.style = {name = #1, rectangle, rounded corners = 1, 
    draw=black, thick, fill=sensitivecolor, inner sep = 5.000000pt, line width = 0.4pt},
  obsout/.style = {name = #1, rectangle, rounded corners = 1, 
    draw=black, thick, fill=outcomecolor, inner sep = 5.000000pt, line width = 0.4pt},
  lat/.style = {name = #1, circle, draw, inner sep = 3.000000pt}
}

\definecolor{violet}{rgb}{0.7,0,0.7}
\definecolor{gray}{rgb}{0.4,0.4,0.4}

\newcommand{\+}[1]{\ensuremath{\mathbf{#1}}}

\newcommand{\cond}{\,\vert\,}

\newcommand{\sM}{\mathcal{M}}
\newcommand{\sG}{\mathcal{G}}

\newcommand{\sC}{\mathcal{C}}

\newcommand{\sW}{\mathcal{W}}

\newcommand{\Pa}[1][]{%
  \ifthenelse{ \equal{#1}{} }
    {\mathrm{Pa}}
    {\mathrm{Pa}_{#1}}
}
\newcommand{\Ch}[1][]{%
  \ifthenelse{ \equal{#1}{} }
    {\textrm{Ch}}
    {\textrm{Ch}_{#1}}
}
\newcommand{\An}[1][]{%
  \ifthenelse{ \equal{#1}{} }
    {\textrm{An}}
    {\textrm{An}_{#1}}
}
\newcommand{\De}[1][]{%
  \ifthenelse{ \equal{#1}{} }
    {\textrm{De}}
    {\textrm{De}_{#1}}
}
\newcommand{\Ne}[1][]{%
  \ifthenelse{ \equal{#1}{} }
    {\textrm{Ne}}
    {\textrm{Ne}_{#1}}
}
\newcommand{\Co}[1][]{%
  \ifthenelse{ \equal{#1}{} }
    {\textrm{Co}}
    {\textrm{Co}_{#1}}
}
\newcommand{\rec}[1][]{%
  \ifthenelse{ \equal{#1}{} }
    {\textrm{Re}}
    {\textrm{Re}_{#1}}
}
\newcommand{\emi}[1][]{%
  \ifthenelse{ \equal{#1}{} }
    {\textrm{Em}}
    {\textrm{Em}_{#1}}
}
\newcommand{\Er}[1]{%
    {\overline{U}_{#1}}
}
\newcommand{\er}[1]{%
    {\overline{u}_{#1}}
}
\newcommand{\erp}[1]{%
    {\xi_{#1}}
}
\newcommand{\Pau}[1][]{%
  \ifthenelse{ \equal{#1}{} }
  {\mathrm{Pa}^\circ}
  {\mathrm{Pa}^\circ_{#1}}
}

\newcommand{\doo}{\mathrm{do}}

\definecolor{colA}{RGB}{241,86,63} 
\definecolor{colB}{RGB}{0,82,174} 
\definecolor{colC}{RGB}{129,103,0}

\newcommand{\citepp}[1]{\citep{#1}}
\newcommand{\citett}[1]{\citet{#1}}

\newtheorem{theorem}{Theorem}
\newtheorem{corollary}{Corollary}
\newtheorem{definition}{Definition}
\newtheorem{lemma}{Lemma}

\newcommand{\N}{\textrm{N}}
\newcommand{\errorterm}{dedicated error term}

\newcommand{\R}{\mathbb{R}}
\newcommand{\E}{\mathbb{E}}
\newcommand{\ud}{\mathrm{d}}
\renewcommand{\vec}[1]{\mathbf{#1}}
\newcommand{\uarg}{\;\cdot\;}

\newcommand{\blind}{1}

\begin{document}

\title{\vspace{-2.5cm}  Simulating counterfactuals}

\author{Juha Karvanen, Santtu Tikka, Matti Vihola\\
Department of Mathematics and Statistics\\
University of Jyvaskyla, Finland\\
}

\date{}

\maketitle

\begin{abstract}
Counterfactual inference considers a hypothetical intervention in a parallel world that shares some evidence with the factual world. If the evidence specifies a conditional distribution on a manifold, counterfactuals may be analytically intractable.
We present an algorithm for simulating values from a counterfactual distribution where conditions can be set on both discrete and continuous variables. We show that the proposed algorithm can be presented as a particle filter leading to asymptotically valid inference. The algorithm is applied to fairness analysis in credit-scoring.\\
~\\
{\it Keywords:}  Causality, Fairness, Particle filter, Sequential Monte-Carlo, Structural causal model
\end{abstract}

%
%

\section{Introduction}
A counterfactual distribution is the probability distribution of a random variable under a hypothetical scenario that differs from the observed reality. ``What would have been the outcome for this individual if they had received a different treatment?'' is an example of a counterfactual question. Here the personal data of the individual constitute the evidence that specifies the observed reality, and the interest lies in the distribution of the outcome under a hypothetical treatment. 

Counterfactual questions belong to the third and highest level in the causal hierarchy \citepp{shpitser2008complete} and are in general more difficult than associational (first level) or interventional (second level) questions.
Algorithms for checking the identifiability of counterfactual queries from observational and experimental data have been developed \citepp{shpitser2007counterfactuals,shpitser2018identification,correa2021nested}  and implemented \citepp{tikka2022identifying}. In many practical cases, the queries may be non-identifiable \citepp{wu2019pc-fairness}. 

Counterfactuals are often linked with questions about fairness, guilt, and responsibility. Notably, the fairness of prediction models has become a major concern in automated decision-making \citepp{wachter2018counterfactual,pessach2022review} where the requirement for fairness often arises directly from the legislation. A classic example of the fairness requirement in decision-making is credit-scoring where the bank's decision to lend must not depend on sensitive variables such as gender or ethnicity \citepp{bartlett2022consumer}.

Several definitions and measures of fairness have been proposed \citepp{mehrabi2021survey,caton2020fairness,carey2022causal}. Among these proposals, counterfactual definitions of fairness \citepp{kusner2017counterfactual,nabi2018fair,chiappa2019path,wu2019pc-fairness,richens2022counterfactual} are intuitively appealing as they compare whether an individual decision would remain the same in a hypothetical counterfactual world. Note that the term ``counterfactual explanations'' \citepp{guidotti2022counterfactual} is sometimes used in the literature on explainable artificial intelligence (XAI)  and interpretable machine learning \citepp{burkart2021survey} in contexts where the term ``contrastive explanations'' proposed by \citett{karimi2021algorithmic} would be more appropriate.

We consider the problem of simulating observations from a specified counterfactual distribution. Given a structural causal model (SCM) and a counterfactual of interest, the counterfactual distribution can be derived in three steps \citepp{pearl:book2009}. First, the distribution of the latent background variables is updated given the evidence expressed as a set of conditions on the observed variables. Second, the causal model is modified according to the hypothetical intervention. Third, the counterfactual distribution of the target variable is calculated in the model that represents the counterfactual scenario under the updated distribution. These three steps require the full knowledge of the causal model in the functional form, i.e. the structural equations and the distributions of the background variables must be known. 

The problem of simulating counterfactuals is similar to the problem of simulating observations from a given distribution. The challenges are related to the first step of counterfactual inference which requires determining a multivariate conditional distribution. This can be done analytically only in special cases where, for instance, all the variables are either discrete or normally distributed.  Evidence that includes continuous variables leads to a conditional distribution concentrated on a manifold, which is generally difficult to simulate. 

In this paper, we present an algorithm for simulating counterfactual distributions. The algorithm is applicable in settings where the causal model is fully known in a parametric form and the structural equations for continuous variables have an additive error term or, more generally, are strictly monotonic with respect to an error term.  The algorithm is essentially tuning-free. Unlike \citett{karimi2021algorithmic} and \citett{javaloy2023causal}, we allow the SCM to contain background variables that affect two or more observed variables, which makes it significantly harder to simulate the counterfactual distribution.

The algorithm processes the conditioning variables one by one in a topological order and obtains an approximate sample from the updated distribution of the background variables (step 1). The challenge of continuous conditioning variables is overcome by using binary search to find solutions that satisfy the conditions and then applying sequential Monte Carlo to calibrate the distribution of these solutions. Discrete conditioning variables are processed by resampling observations that satisfy the condition. Next, the causal model is  modified (step 2) and a sample from the counterfactual distribution is acquired by simulating the modified causal model with the obtained sample of the background variables (step 3). 

We show that the conditional simulation at step 1 can be interpreted as a particle filter/sequential Monte Carlo \citep[e.g.,][]{gordon-salmond-smith,doucet2000sequential,del-moral,cappe-moulines-ryden,chopin2020introduction}.  
Theoretical results from the particle filter literature guarantee good asymptotic properties of the sample. In particular, we state a mean square error convergence rate and a central limit theorem for samples obtained with the proposed algorithm.

In real-world applications, the full knowledge of the causal model may not be available. Despite this serious restriction, we argue that the simulation-based counterfactual inference may still have its role in fairness evaluation. Consider a prediction model that is used for decision-making in a situation where the underlying causal model is unknown. To ensure fairness in this situation, the prediction model should be fair under any reasonable causal model. This implies that an analyst evaluating the fairness of the prediction model could choose some reasonable causal models for the evaluation. Deviations from fairness under any of these causal models indicate that the prediction model is not fair in general. 

We use the counterfactual simulation algorithm as the main component in a fairness evaluation algorithm of prediction models. We simulate data from specified counterfactual distributions and compare the predictions for these settings.  We demonstrate with a credit-scoring example how the fairness evaluation algorithm can be applied to opaque AI models 
without access to real data. As a result, we find out if an AI model is fair in the tested setting, and if not, learn how large differences there are in the outcome under different counterfactual interventions.

The rest of the paper is organized as follows. The notation and the basic definitions are given in Section~\ref{sec:notation}. In Section~\ref{sec:algorithm}, the counterfactual simulation algorithm and the fairness evaluation algorithm are introduced, and the conditional simulation is presented as a particle filter. In Section~\ref{sec:simulations}, the performance of the simulation algorithm is tested in benchmark cases where the counterfactual distributions can be derived analytically. In Section~\ref{sec:fairness}, the fairness evaluation algorithm is applied to a credit-scoring example. Section~\ref{sec:discussion} concludes the paper.

\section{Notation and Definitions} \label{sec:notation}

We use uppercase letters to denote variables, lowercase letters to denote values, and bold letters to denote sets of variables or values. The primary object of counterfactual inference is a causal model \citepp{pearl:book2009}:

\begin{definition}[SCM] \label{def:scm}
A structural causal model $\sM$ is a tuple $(\+ V, \+U, \+ F, p(\+ u))$, where 
\begin{enumerate}[(1) ]
  \item $\+ V$ is a set of observed (endogenous) variables that are determined by other variables in the model.
  \item $\+ U$ is a set of background variables that are unobserved random variables.
  \item $\+ F$ is a set of functions $\{f_V \mid V \in \+ V\}$ such that each $f_V$ is a mapping from \mbox{$\+ U \cup (\+ V \setminus \{V\})$} to $V$ and such that $\+ F$ forms a mapping from $\+ U$ to $\+ V$. Symbolically, the set of equations $\+ F$ can be represented by writing $V = f_V(\Pa^*(V))$, where $\Pa^*(V) \subset \+ U \cup (\+ V \setminus \{V\})$ is the unique minimal set of variables sufficient for representing $f_V$. 
  \item $p(\+ u)$ is the joint probability distribution of $\+ U$.
\end{enumerate}
\end{definition}

An SCM is associated with a directed graph $\sG$ where there is an edge $W \rightarrow V$ if and only if $W \in \Pa^*(V)$. In other words, $\Pa^*(V)$ denotes the (observed and unobserved) parents of $V$ in graph $\sG$. We will consider only recursive SCMs where the set $\+F$ defines a \emph{topological order} of the variables $\+V \cup \+U$, i.e., an order where $\Pa^*(Z) < Z$ for all $Z \in \+ V \cup \+ U$. The graph associated with a recursive SCM is a directed acyclic graph (DAG). Notation $\tau(V)$ refers to the variables that precede $V$ in the topological order meaning variables $W \in \+ V \cup \+ U$ such that $W < V$. We assume that variables $\+U$ precede $\+V$ in the topological order. Notation $\Pa(V)$ is a shorthand notation for $\Pa^*(V) \cap \+V$ meaning the observed parents. Similarly, $\An^*(V)$ denotes the ancestors of $V$ in $\sG$ and $\An(V) = \An^*(V) \cap \+V$. Notations $\tau(v)$, $\Pa^*(v)$, $\Pa(v)$, $\An^*(v)$ and $\An(v)$ refer to the values of the variables $\tau(V)$, $\Pa^*(V)$, $\Pa(V)$, $\An^*(V)$ and $\An(V)$, respectively.

Data can be simulated from an SCM by generating the values $\+ u$ of the background variables $\+ U$ from the distribution $p(\+u)$ and then applying the functions $\+F$ in the topological order to obtain values $\+ v$ of $\+ V$. For a data matrix $\+D$ with named columns (variables), the following notation is used: $\+D[i,]$ refers to the $i$th row, $\+D[C=c,]$ where $C$ is a column name of $\+D$ refers to rows where condition $C=c$ is fulfilled, $\+D[,\+S]$ refers to data on variables $\+S$, i.e., columns whose variable names are in the set $\+S$, and $\+D[i,\+S]$ refers to the $i$th row of variables $\+S$. The notation is similar to many high-level programming languages.

An intervention $\doo(\+ X = \+ x)$ targeting SCM $\sM$ induces a \emph{submodel} \citepp{pearl:book2009}, denoted by $\sM_{\doo(\+ X = x)}$, where those functions in $\sM$ that determine the variables $\+X$ are replaced by constant functions that output the corresponding values $\+ x$. We also use the subscript $\doo(\+ X = x)$ to distinguish variables in the submodel from the original variables, e.g., $\+ W_{\doo(\+ X = \+ x)}$ is the set of variables $\+ W$ in $\sM_{\doo(\+ X = x)}$. The joint distribution of $\+ V_{\doo(\+ X = x)}$ in the submodel $\sM_{\doo(\+ X = x)}$ is known as the \emph{interventional distribution} or \emph{causal effect}. The effects of interventions typically relate to interventional considerations such as the effect of a treatment on a response, but they also facilitate the analysis of counterfactuals that consider hypothetical actions that run contrary to what was observed in reality. 

For an SCM where all variables are discrete, a counterfactual distribution can be evaluated by marginalizing over those values of the background variables that result in the specified evidence. More specifically, the probability distribution of a set of observed variables $\+ W \subset \+ V \setminus \+ X$ in the submodel $\sM_{\doo(\+ X = \+ x)}$ conditional on a set of observations (the evidence) $\+ C = \+ c$ such that $\+ C, \+ X \subseteq \+ V$ can be written as
\[
  p(\+ W_{\doo(\+ X = \+ x)} = \+ w \cond \+ C = \+ c) = \sum_{\{\+ u \cond f^\circ_{\+ W}(\+ u) = \+ w\}} p(\+ U = \+u | \+ C = \+ c),
\]
where $f^\circ_{\+ W}$ denotes the functions of $\+ W_{\doo(\+ X = \+ x)}$ in the submodel $\sM_{\doo(\+ X = \+ x)}$ expressed in terms of $\+ U$ (such functions always exist due to property~(3) of Definition~\ref{def:scm}). The distribution can be also expressed as
\[
  p(\+ W_{\doo(\+ X = \+ x)} = \+ w \cond \+ C = \+ c) = \sum_{\+ u} I(f^\circ_{\+ W}(\+ u) = \+ w) p(\+ U = \+u | \+ C = \+ c).
\]
This formulation generalizes to the scenario where the variables $\+ U$ are continuous as follows
\[
  p(\+ W_{\doo(\+ X = \+ x)} \in A \cond \+ C = \+ c) = \int I(f^\circ_{\+ W}(\+ u) \in A) p(\+ U = \+u | \+ C = \+ c)\, \ud\+{u}.
\]
where $A$ is some event.

The submodel $\sM_{\doo(\+ X = \+ x)}$ and the updated distribution $p(\+ U = \+u \cond \+C = \+c)$ together define a \emph{parallel world} where the state of the background variables is shared with the non-interventional world. In contrast to interventional distributions, the sets of variables $\+ C$ and $\+ X$ defining the evidence $\+ C = \+ c$ and the intervention $\doo(\+ X = \+ x)$ need not be disjoint for a counterfactual distribution. If the sets are disjoint, a counterfactual distribution simply reduces to a conditional interventional distribution.

We follow the general procedure introduced by \citett{pearl:book2009} to evaluate counterfactual distributions.
\begin{definition}[Evaluation of a counterfactual distribution] \label{def:counterfactual}
Let $\sM = (\+ V, \+ U, \+ F, p(\+ u))$ be a recursive SCM and let $\+ C, \+ X \subseteq \+ V$ and $\+ W \subseteq \+ V \setminus \+ X$. The counterfactual distribution $p(\+ W_{\doo(\+ X = \+ x)} = \+ w \cond \+ C = \+ c)$ can be evaluated using the following three steps.
\begin{enumerate}[1. ]
  \item Update $p(\+ U = \+ u)$ by the evidence $\+ C = \+c$ to obtain $p(\+ U = \+ u \cond \+ C = \+ c)$ (Bayes' Theorem)
  \item Construct the submodel $\sM_{\doo(\+ X = \+ x)}$ corresponding to the counterfactual scenario.
  \item Use the submodel $\sM_{\doo(\+ X = \+ x)}$ and the updated distribution $p(\+ U = \+ u \cond \+ C = \+ c)$ to compute $p(\+ W_{\doo(\+ X = \+ x)} = \+ w \cond \+ C = \+ c)$.
\end{enumerate}
\end{definition}
In practice, however, this procedure may not be directly applicable because it is often not possible to analytically compute $p(\+ W_{\doo(\+ X = \+ x)} = \+ w \cond \+ C = \+ c)$ in the third step, for example when some of the variables in the conditioning set $\+C$ are continuous. A simple illustration of counterfactual inference in a case where the distribution $p(\+ W_{\doo(\+ X = \+ x)} = \+ w \cond \+ C = \+ c)$ can be derived analytically is presented in Online Appendix 1.

Our objective is to simulate observations from counterfactual distributions given an SCM in the general case where an analytical solution is not available. Before the evaluation of a counterfactual distribution, the SCM can be pruned similarly to what is done by causal effect identification algorithms \citepp{tikka2017pruning} by restricting our attention to only those variables that are relevant to the task. For this purpose, we define an ancestral SCM as follows:

\begin{definition}[Ancestral SCM] \label{def:ancestral}
Let $\sM = (\+ V, \+ U, \+ F, p(\+ u))$ be a recursive SCM and let $\+ Z \subseteq \+ V \cup \+ U$. Then $\sM_{\+ Z} = (\+ V', \+ U', \+ F', p(\+ u'))$ is the ancestral SCM of $\sM$ with respect to $\+ Z$ where $\+V'=\+V \cap (\An(\+Z) \cup \+Z)$, $\+U'=\+U \cap (\An^*(\+Z) \cup \+Z)$, and $\+F'$ is the subset of $\+F$ that contains the functions for $\+V'$.
\end{definition}
To evaluate a counterfactual distribution, it suffices to restrict our attention to the relevant ancestral SCM.
\begin{theorem} \label{th:pruning_equivalence}
Let $\sM = (\+ V, \+ U, \+ F, p(\+ u))$ be a recursive SCM and let $\eta = p(\+ W_{\doo(\+ X = \+ x)} = \+ w \cond \+ C = \+ c)$ be a counterfactual distribution such that $\+ W \cup \+ X \cup \+ C \subseteq \+ V$. Then $\eta$ evaluated via Definition~\ref{def:counterfactual} in $\sM$ is equivalent to $\eta$ evaluated via Definition~\ref{def:counterfactual} in the ancestral SCM $\sM_{\+ W \cup \+ X \cup \+ C}$.
\end{theorem}
\begin{proof}
Consider a variable $Y \notin \An^*(\+W \cup \+X \cup \+C)$. As $Y$ is not an ancestor of $\+W$, $\+X$, or $\+C$, it does not have any impact in steps 1--3 of Definition~\ref{def:counterfactual}.
\end{proof}
This type of pruning is a useful operation because it often reduces the number of variables in the SCM and thus allows for more efficient simulation. 

The main application of counterfactual simulation is counterfactual fairness for which different definitions have been proposed. Let $S$ be a sensitive variable (sometimes referred to as a protected variable), $Y$ the outcome of interest, and $\widehat{Y}$ an estimator of $Y$. \citett{kusner2017counterfactual} define $\widehat{Y}$ to be counterfactually fair if
\begin{equation} \label{eq:kusnerfairness}
p(\widehat{Y}_{\doo(S=s)} \cond \+X=\+x,S=s) = p(\widehat{Y}_{\doo(S=s')} \cond \+X=\+x,S=s)
\end{equation}
under any context $\+X=\+x$ and for any values $s$ and $s'$ of $S$. In other words, the evidence $\+X=\+x$ and $S=s$ has been observed and in this situation, the distribution of $\widehat{Y}$ should be the same under the intervention $\doo(S=s)$ and the counterfactual intervention $\doo(S=s')$.

Other authors \citepp{nabi2018fair,chiappa2019path} have found the definition of Equation~\eqref{eq:kusnerfairness} too restrictive and have instead considered path-specific counterfactual inference where the sensitive variable may affect the decision via a fair pathway. For instance, even if gender affects education, it would be fair to use education in a recruitment decision. However, it would be discriminatory to base the recruitment decision on gender or its proxies such as the given name. The recognition of fair and unfair pathways requires an understanding of the underlying causal mechanisms. Along these lines, we will use the following definition.

\begin{definition}[Counterfactual fairness] \label{def:fairness}
Let $(\+V,\+U,\+F,p(\+u))$ be an SCM where $\+S \subset \+V$ is the set of sensitive variables and $\+Y \subset \+V$ is the set of outcome variables. Define $\+W = \Pa(\+Y) \setminus \+S$ and $\+Z=\+V \setminus (\+Y \cup \+W \cup \+S)$, and let the conditioning variables be $\+ C = \+ W \cup \+ Z \cup \+ S$. The decisions based on an estimator $\widehat{\+Y}$ are counterfactually fair if
\[
p(\widehat{\+Y}_{\doo(\+S=\+s, \+W=\+w)} \cond \+W=\+w,\+Z=\+z,\+S=\+s) = p(\widehat{\+Y}_{\doo(\+S=\+s', \+W=\+w)} \cond \+W=\+w,\+Z=\+z,\+S=\+s)
\]
under any context $\+W=\+w$, $\+Z=\+z$, and for any values $\+s$ and $\+s'$ of $\+S$.
\end{definition}
Definition~\ref{def:fairness} fixes the non-sensitive parents $\+W$ of the outcome to their observed values under any counterfactual intervention of the sensitive variables. In the recruitment example, this would mean that a counterfactual intervention on gender could change the given name, but education would remain unchanged because it has a direct effect on job performance. This kind of definition is mentioned by \citett{wu2019pc-fairness} who call it ``individual indirect discrimination''.

\section{Algorithms} \label{sec:algorithm}

We proceed to construct an algorithmic framework for counterfactual simulation and fairness evaluation. The general principle of the counterfactual simulation is straightforward: the proposed simulation algorithm randomly generates the values of some background variables and numerically solves the rest of the background variables from the conditions. The obtained candidate observations are iteratively resampled with unequal probabilities so that the selected sample is an approximate random sample from the conditional distribution of the background variables. The counterfactual distribution is approximated using this sample of background variables together with the intervened SCM.

This section is organized as follows. First, the assumptions and some basic results are presented in Section~\ref{subsec:assumptions}.  
Next, the formal algorithms are presented in Sections~\ref{subsec:conditionalalgorithms} and \ref{subsec:inferencealgorithms}. Finally, in Section~\ref{subsec:theoretical} we show that the simulation algorithm leads to asymptotically consistent inference and can be interpreted as a sequential Monte Carlo algorithm.  

\subsection{Assumptions and Basic Results} \label{subsec:assumptions}

The starting point of counterfactual simulation is an SCM $\sM=(\+ V, \+U, \+ F, p(\+ u))$ where the functions $\+F$ and the distribution $p(\+u)$ are fully known, and the objective is to simulate values from counterfactual distribution $p(\+ W_{\doo(\+ X = \+ x)} = \+ w \cond \+ C = \+ c)$.
In order to simulate counterfactuals in cases where the set of conditioning variables $\+C$ contains continuous variables, we refine the definition of the SCM by imposing some restrictions on background variables. First, we define a special type of background variable:

\begin{definition}[Dedicated error term]
In an SCM $\sM=(\+V,\+U,\+F,p(\+u))$, a background variable $U \in \+U$ is a \errorterm{} if it is independent of the other background variables and has exactly one child. If an observed variable $V \in \+ V$ has only one \errorterm{}, it is denoted by $\Er{V}$, and the notation $\erp{V}(u)=p_{\Er{V}}(u)$ is used for the density function of $\Er{V}$.
\end{definition} 
With this definition, the background variables can be divided into two categories, $\overline{\+U}$ and $\widetilde{\+U}$, where $\overline{\+U}$ is the set of \errorterm{}s and $\widetilde{\+U}$ is the set of other background variables which we call global background variables henceforth, and $\+U = \overline{\+U} \cup \widetilde{\+U}$. Further, let $\overline{\+U}_{\+C}$ and $\overline{\+U}_{\+V \setminus \+C}$ represent the \errorterm{}s of variables in $\+C$ and in $\+V \setminus \+C$, respectively. From the independence of \errorterm{}s, it follows that 
\[
  p(\overline{\+u}) = \prod_{U \in \overline{\+U}} p_{U}(u).
\]
The variables $\widetilde{\+U}$ act as unobserved confounders while each \errorterm{} only affects the variation of one specific observed variable. We will often consider the  \errorterm{}s separately from other parents of a variable and for this purpose, we introduce the notation $\Pau(V)=\Pa^*(V) \cap (\+V \cup \widetilde{\+U}) =  \Pa^*(V) \setminus \{\Er{V}\}$.

Our SCMs of interest have \errorterm{}s with a monotonic effect in their respective functions.
\begin{definition}[u-monotonic SCM] \label{def:u-monotonic}
A recursive SCM $\sM=(\+V,\+U,\+F,p(\+u))$ is u-monotonic with respect to a set of continuous variables $\+W \subseteq \+V$ if it has the following properties:  
\begin{enumerate}[(a) ]
\item Each $V \in \+V$ is univariate and has exactly one \errorterm{} that is a continuous univariate random variable. 
\label{assump:one_errorterm}
\item For all  $W \in \+W$, it holds that  given $\Pau(W)$, the value of $W$ is determined as $w=g_W(\er{w}) = g_W(\er{w};\Pau(w))=f_W(\er{w},\Pau(w))$ where  $g_W$ is a continuous, differentiable and strictly increasing function of \errorterm{} $\er{w}$, and the parameters of the function are $\Pau(W)$. \label{assump:continuous_monotonic} 
\end{enumerate}
\end{definition}
The inverses (in their domain) and derivatives of the function $g_W$ exist and they are denoted by $g_W^{-1}$ and $g_W'$, respectively. Cases where  $g_W$ would be a strictly decreasing function can be accommodated to Definition \ref{def:u-monotonic} by redefining $\overline{u}_w^{\textrm{new}}=-\er{w}$ and $g_W^{\textrm{new}}(u)=g_W(-u)$. The second property of Definition~\ref{def:u-monotonic} is trivially fulfilled in the important special case, the additive noise model 
\[
  w = f_W(\er{w},\Pau(w)) = f_W^*(\Pau(w)) + \er{w},
\]
where $f_W^*$ is some function that does not depend on $\er{w}$. 
The formulation of Definition~\ref{def:u-monotonic} permits $g_W(\er{w};\Pau(w))$ to be a complicated function, which allows for heteroscedastic noise models. 

The first step of Definition~\ref{def:counterfactual} involves deriving the conditional distribution \mbox{$p(\+U=\+u \cond \+C=\+c)$}. The following results characterize conditional distributions in a u-monotonic SCM.

\begin{lemma} \label{lem:dedicated-transformation_new}
Let $W$ be a continuous variable in a u-monotonic SCM with respect to a set $\+W$ such that $W \in \+W$. The conditional density of $W \cond (\Pau(W)=\Pau(w))$ is
\[
  p\big(w\cond \Pau(w)\big) 
  = \begin{cases}
   \displaystyle
    \frac{\erp{W}(\er{W=w})}{g_W'(\er{W=w}; \Pau(w))}, & \text{if }a < w < b, \\
    0, & \text{otherwise},
  \end{cases}
\]
where $(a,b)$ is the domain of $g_{W}(\uarg; \Pau(w))$ and $\er{W=w} = g_{W}^{-1}(w; \Pau(w))$. The notation $\omega_W(\er{W=w})=p\big(w\cond \Pau(w)\big)$ is used to assign a weight for $\er{W=w}$. 
\end{lemma}
\begin{proof}
The conditional density of $\Er{W}\mid \Pau(W)$ is, by independence, $\erp{W}$, the unconditional density of $\Er{W}$. The conditional distribution $W \mid \Pau(W)$ is a point mass at $g_W(\Er{W}; \Pau(W))$.
Therefore, by the standard transformation formula, $W$ has the density
\[
  p_W(g^{-1}(w)) (g^{-1})'(w) = \frac{\erp{W}(g^{-1}(w))}{g'(g^{-1}(w))},
\]
for all $w$ in its domain, and zero elsewhere.
\end{proof}

Lemma~\ref{lem:dedicated-transformation_new} tells that the value of the \errorterm{} $\er{W=w}$ is determined by the value $w$ and the values of other parents of $W$ via the function $g_{W}^{-1}(w; \Pau(w))$. The parents $\Pau(w)$ may include global background variables $\widetilde{\vec{U}}$. When the parents $\Pau(w)$ are fixed, the distribution of $W$ is obtained from the distribution $\er{W=w}$ by the standard transformation formula. 
Note that for an additive noise model, we simply have $\omega_W(\er{W=w})=\erp{W}(\er{W=w})$. The next lemma describes the conditional distribution of the background variables.

\begin{lemma} \label{lem:unobserved-given-evidence_new}
In a u-monotonic SCM, the conditional density of $\widetilde{\vec{U}},\vec{\+U}_{\+V \setminus \+C} \cond (\+C = \+c)$ satisfies:
\[
  p(\widetilde{\vec{u}}, \vec{u}_{\+V \setminus \+C} \cond \+c) \propto 
    p(\widetilde{\vec{u}}) \left(\prod_{U \in \overline{\+U}_{\+V \setminus \+C}} p(u)\right)\left(\prod_{C \in \+C}  \omega_C(\er{C=c}) \right),
\]
where terms $ \omega_C(\er{C=c})=p\big(c\cond \Pau(c)\big)$ in the last product are defined in Lemma \ref{lem:dedicated-transformation_new}. 
\end{lemma}
\begin{proof}
The expression on the right is the joint density of $(\widetilde{\+U}, \overline{\+U}_{\+V \setminus \+C}, \+C)$, written by the chain rule:
\[
  p(\widetilde{\vec{u}}) \left(\prod_{U \in \overline{\+U}_{\+V \setminus \+C}} p(u \cond \widetilde{\vec{u}}, \tau(u) \cap \+u_{\+V \setminus \+C}) \right)
  \left(\prod_{C\in \+C} p(c \cond \widetilde{\vec{u}}, \overline{\+u}_{\+V \setminus \+C}, \tau(c) \cap \+c)\right).
\]
The claim follows by applying the mutual independence of \errorterm{}s to the first product and Lemma \ref{lem:dedicated-transformation_new} to the second product.
\end{proof}

Combining the results above, the conditional distribution for all background variables can be written as follows:
\begin{corollary} \label{cor:all-unobserved_new}
The conditional distribution of $(\widetilde{\vec{U}},\overline{\+U},\+Y) \cond (\+C=\+c)$, where $\+Y=\+V \setminus \+C$, is:
\[
  p(\widetilde{\vec{u}}, \overline{\+u}_{\+V \setminus \+C} \cond \+c) \,\ud \widetilde{\vec{u}} \,\ud \overline{\+u}_{\+V \setminus \+C}
  \left(\prod_{C \in \+C} I(g^{-1}_{C}(c;\Pau(c)) \in \ud u_{c} )\right)\!
  \left(\prod_{Y \in \+Y} I( g_{Y}(u_{Y};\Pau(y)) \in \ud y)\right),
\]
where $I(\cdot)$ is an indicator function. In other words, $\widetilde{\vec{u}}$ and $\overline{\+u}_{\+V \setminus \+C}$ have a density, and $\overline{\+u}_{\+C}$ and $\+y$ depend on them deterministically as above.
\end{corollary}

\subsection{Algorithms for Conditional Simulation} \label{subsec:conditionalalgorithms} 
Next, we will consider algorithms that simulate observations from a u-monotonic SCM under given conditions (step 1 of Definition~\ref{def:counterfactual}). First, we will present simulation algorithms for the cases with a single conditioning variable. Algorithm~\ref{alg:continuous_condition} considers the case of a continuous conditioning variable and Algorithm~\ref{alg:discrete_condition} the case of a discrete conditioning variable. Algorithm~\ref{alg:multiple_conditions} processes multiple conditioning variables and calls Algorithms~\ref{alg:continuous_condition} and \ref{alg:discrete_condition}. The full workflow of counterfactual inference is implemented later in Algorithm~\ref{alg:simulate_counterfactual}. Algorithm~\ref{alg:evaluate_fairness} applies Algorithm~\ref{alg:simulate_counterfactual} in fairness evaluation.

\subsubsection{Continuous Condition}

We describe the operation of Algorithm~\ref{alg:continuous_condition}. On lines~\ref{line:call_simulatescm_conditional} and \ref{line:call_simulatescm_unconditional}, the procedure \textsc{SimulateSCM} is called. The unconditional call on line~\ref{line:call_simulatescm_unconditional} first simulates $n$ realizations of the background variables from $p(\+u)$ and then applies functions $\+F$ to obtain the values of the observed variables $\+V$. The conditional version on line~\ref{line:call_simulatescm_conditional} is similar, but the values of the variables in $\+D_0$ are not generated but taken directly from $\+D_0$. As a result of line~\ref{line:call_simulatescm_conditional} or line~\ref{line:call_simulatescm_unconditional}, the data matrix $\+D$ contains $n$ rows and the values of all variables in $\+V$ and $\+U$. In practice, it is not necessary to simulate the descendants of $\Er{C}$ on lines~\ref{line:call_simulatescm_conditional} or \ref{line:call_simulatescm_unconditional} because these variables will be updated later.

\begin{algorithm}[!ht]
  \begin{algorithmic}[1]
  \Function{SimulateContinuousCondition}{$n$, $\sM$, $C=c$, $\+D_0$}
  \If{$\+D_0$ exists} 
    \State $\+D \gets$ \Call{SimulateSCM}{$n$, $\mathcal{M}$, $\+D_0$} \label{line:call_simulatescm_conditional}
  \Else
    \State $\+D \gets$ \Call{SimulateSCM}{$n$, $\mathcal{M}$, $\emptyset$} \label{line:call_simulatescm_unconditional}
  \EndIf
  \For{$i=1$ to $n$} \label{line:loop_rows}
    \State $\+D[i,\Er{C}] \gets$ \Call{FindRoot}{$\mathcal{M}$, $\+D[i,]$, $\Er{C}$, $C$, $c$} \label{line:call_findroot}
    \State $\boldsymbol{\omega}[i] \gets \omega_C(\+D[i,\Er{C}])$ \label{line:calculate_weight}
   \EndFor
  \State $\+D \gets$ \Call{Sample}{$\+D$, $n$, $\boldsymbol{\omega}$} \label{line:call_sample}
    \State $\+D \gets$ \Call{SimulateSCM}{$n$, $\mathcal{M}$, $\+D[,\+U]$} \label{line:call_simulatescm_condU}
    \State \Return $\+D$ 
  \EndFunction
  \end{algorithmic}
  \caption{An algorithm for simulating $n$ observations from a u-monotonic SCM $\sM=(\+V,\+U,\+F,p(\+u))$ on the condition that the value of a continuous variable $C$ is $c$. 
  The optional argument $\+D_0$ is an $n$-row data matrix containing the values of some variables $\+V_0 \subset \+V$, $\+U_0 \subset \+U$ that precede $C$ in the topological order and have been already fixed.} 
  \label{alg:continuous_condition}
\end{algorithm}

Starting from line~\ref{line:loop_rows}, the values of the error term $\Er{C}$ in the data matrix $\+D$ are modified so that the condition $C=c$ is fulfilled. The procedure \textsc{FindRoot} on line~\ref{line:call_findroot} uses binary search to find the value $u_i$ of $\Er{C}$ so that
\begin{equation} \label{eq:solve_ui}
f_C(u_i,\+D[i,\Pau(C)]) = c,
\end{equation}
where $f_C \in \+F$ is the function that determines the value $C$ in the u-monotonic SCM $\sM$. Due to the monotonicity assumption, there is at most one value $u_i$ that solves Equation~\eqref{eq:solve_ui}, and if the solution exists it can be found by binary search. If a solution cannot be found, $\+D[i,\Er{C}]$ is set as ``not available''. The binary search is not needed for additive noise models, where $u_i$ can be solved analytically. On line~\ref{line:calculate_weight}, the sampling weights are calculated with function $\omega_C$ defined in Lemma~\ref{lem:dedicated-transformation_new}.   
If $\+D[i,\Er{C}]$ is not available, $\boldsymbol{\omega}[i]$ will be set as $0$.

On line~\ref{line:call_sample}, a re-sample is drawn with the weights $\boldsymbol{\omega}=(\boldsymbol{\omega}[1],\ldots,\boldsymbol{\omega}[n])$ (it is assumed here that at least one of the weights is positive). The sampling is carried out with replacement which leads to conditionally independent but non-unique observations. 
Uniqueness could be achieved by adding low-variance noise to the background variables before line~\ref{line:call_simulatescm_condU} but this would lead to a small divergence from the required condition $C = c$. On line~\ref{line:call_simulatescm_condU}, the observed variables of $\sM$ are updated because the change in $\Er{C}$ affects its descendants.

\subsubsection{Discrete Condition}

Next, we consider the case of a single discrete conditioning variable and describe the operation of Algorithm~\ref{alg:discrete_condition}. 
Lines~\ref{line:discrete_call_simulatescm_conditional} and \ref{line:discrete_call_simulatescm_unconditional} are similar to lines \ref{line:call_simulatescm_conditional} and \ref{line:call_simulatescm_unconditional} of Algorithm~\ref{alg:continuous_condition}, respectively. On line~\ref{line:discrete_call_sample}, $n$ observations are sampled with replacement from those observations that fulfill the target condition $C=c$.
\begin{algorithm}[!ht]
  \begin{algorithmic}[1]
   \Function{SimulateDiscreteCondition}{$n$, $\sM$, $C=c$, $\+D_0$}
  \If{$\+D_0$ exists} 
    \State $\+D \gets$ \Call{SimulateSCM}{$n$, $\mathcal{M}$, $\+D_0$} \label{line:discrete_call_simulatescm_conditional}
  \Else
    \State $\+D \gets$ \Call{SimulateSCM}{$n$, $\mathcal{M}$, $\emptyset$} \label{line:discrete_call_simulatescm_unconditional}
  \EndIf   
	\State $\+D \gets$ \Call{Sample}{$\+D[C=c,]$, $n$} \label{line:discrete_call_sample}
   \State \Return $\+D$
  \EndFunction
  \end{algorithmic}
  \caption{An algorithm for simulating $n$ observations from causal model $\sM=(\+V,\+U,\+F,p(\+u))$ on the condition that the value of discrete variable $C$ is $c$. 
  The optional argument $\+D_0$ is an $n$-row data matrix containing the values of some variables $\+V_0 \subset \+V$, $\+U_0 \subset \+U$ that precede $C$ in the topological order and have been already fixed.}
  \label{alg:discrete_condition}
\end{algorithm}

\subsubsection{Multiple Simultaneous Discrete and Continuous Conditions}

When multiple conditions are present in the counterfactual, sequential calls of Algorithms~\ref{alg:continuous_condition} and \ref{alg:discrete_condition} are needed. Algorithm~\ref{alg:multiple_conditions} presents the required steps where the type of the conditioning variable decides whether Algorithm \textsc{SimulateContinuousCondition} or \textsc{SimulateDiscreteCondition} is called. On lines~\ref{line:call_simulatecontinuouscondition_first} or \ref{line:call_simulatediscretecondition_first}, the data are simulated according to the condition that is the first in the topological order. Starting from line~\ref{line:loop_conditions}, the remaining conditions are processed in the topological order. On line~\ref{line:set_ancestors}, the set $\+S$ contains variables that have already been processed. On lines~\ref{line:call_simulatecontinuouscondition} or \ref{line:call_simulatediscretecondition}, the data are simulated according to each condition in the topological order taking into account that the variables in $\+S$ have been already fixed in $\+D$.

\begin{algorithm}[!ht]
  \begin{algorithmic}[1]
  \Function{SimulateMultipleConditions}{$n$, $\sM$, $\sC$}
    \If{$C_1$ is continuous}
      \State $\+D \gets$ \Call{SimulateContinuousCondition}{$n$, $\sM$, $C_1 = c_1$} \label{line:call_simulatecontinuouscondition_first}
    \Else
      \State $\+D \gets$ \Call{SimulateDiscreteCondition}{$n$, $\sM$, $C_1 = c_1$, $n$} \label{line:call_simulatediscretecondition_first}
    \EndIf
    \For{$k=2$ to $K$} \label{line:loop_conditions}
      \State  $\+S \gets (C_1 \cup \An^*(C_1)) \cup \cdots \cup (C_{k-1} \cup \An^*(C_{k-1}))$ \label{line:set_ancestors}
      \If{$C_k$ is continuous}
        \State $\+D \gets$ \Call{SimulateContinuousCondition}{$n$, $\sM$, $C_k = c_k$, $\+D[,\+S]$} \label{line:call_simulatecontinuouscondition}
      \Else
        \State $\+D \gets$ \Call{SimulateDiscreteCondition}{$n$, $\sM$, $C_k = c_k$, $n$, $\+D[,\+S]$} \label{line:call_simulatediscretecondition}
      \EndIf
    \EndFor
    \State \Return $\+D$ \label{line:return_multiple_data}
  \EndFunction
  \end{algorithmic}
  \caption{An algorithm for simulating $n$ observations from a u-monotonic SMC $\sM=(\+V,\+U,\+F,p(\+u))$ with respect to all continuous variables in $C_1,\ldots,C_K$ under the conditions $\sC = (C_1=c_1) \land \cdots \land (C_K = c_K)$. The topological order of the variables in the conditions is $C_1 < C_2 < \cdots < C_K$. The batch size $n^*=n$ is used in Algorithm~\ref{alg:discrete_condition}.}
  \label{alg:multiple_conditions}
\end{algorithm}

\subsection{Algorithms for Counterfactual Inference and Fairness} \label{subsec:inferencealgorithms}

Next, we present high-level algorithms that use Algorithm~\ref{alg:multiple_conditions} to simulate data from a multivariate conditional distribution. Algorithm~\ref{alg:simulate_counterfactual} simulates observations from a counterfactual distribution. The input consists of a u-monotonic SCM, conditions that define the situation considered, an intervention to be applied, and the number of observations. The algorithm has three steps that correspond to the three steps of the evaluation of counterfactuals in Definition~\ref{def:counterfactual}. On line~\ref{line:call_simulatemultiplecondition}, data are simulated using Algorithm~\ref{alg:multiple_conditions}. On line~\ref{line:call_intervene}, an intervention $\doo(\+X=\+x)$ is applied. In practice, the functions of variables $\+X$ in the SCM are replaced by constant-valued functions. On line~\ref{line:call_simulatescm_counterfactual}, counterfactual observations are simulated from the intervened SCM with the background variables simulated on line~\ref{line:call_simulatemultiplecondition}.

\begin{algorithm}[!ht]
  \begin{algorithmic}[1]
  \Function{SimulateCounterfactual}{$\sM$, $\sC$, $\doo(\+X=\+x)$, $n$}
    \State $\+D_0 \gets$ \Call{SimulateMultipleConditions}{$n$, $\sM$, $\sC$}  \label{line:call_simulatemultiplecondition}
    \State $\sM_{\+x} \gets$ \Call{Intervene}{$\sM$, $\doo(\+X=\+x)$} \label{line:call_intervene}
    \State $\+D \gets$ \Call{SimulateSCM}{$n$, $\sM_{\+x}$, $\+D_0[,\+U]$} \label{line:call_simulatescm_counterfactual}
    \State \Return $\+D$ \label{line:return_counterfactual_data}
  \EndFunction
  \end{algorithmic}
  \caption{An algorithm for simulating $n$ observations from a counterfactual distribution under the conditions $\sC = (C_1=c_1) \land \cdots \land (C_K = c_K)$ in an SCM $\sM=(\+V,\+U,\+F,p(\+u))$ that is u-monotonic with respect to all continuous variables in the set $\{C_1,\ldots,C_K\}$. The topological order of the variables in the conditions is $C_1 < C_2 < \cdots < C_K$.} 
  \label{alg:simulate_counterfactual}
\end{algorithm}

Algorithm~\ref{alg:evaluate_fairness} evaluates the fairness of a prediction model based on Definition~\ref{def:fairness}. The input consists of a prediction model, an SCM, a set of sensitive variables, a case to be evaluated, and the size of the data to be simulated. The prediction model can be an opaque AI model where the user is unaware of the functional details of the model and has only access to the predictions.

\begin{algorithm}[!ht]
  \begin{algorithmic}[1]
  \Function{EvaluateFairness}{$\widehat{Y}(\cdot)$, $\sM$, $\+S$, $\sW \land \sC$, $n$}
    \For{all $\+s$ in $\textrm{dom}(\+S)$} \label{line:loop_values}
      \State $\+D_{\+s} \gets$ \Call{SimulateCounterfactual}{$n$, $\sM$, $\sW \land \sC$, $\doo(\+S=\+s, \sW)$} \label{line:call_simulatecounterfactual}
      \State $\widehat{\+Y}_{\+s} \gets \widehat{Y}(\+D_{\+s})$ \label{line:call_predictionmodel}
    \EndFor
    \State \Return \Call{CheckResponse}{$\{ \widehat{\+Y}_{\+s}\}$} \label{line:call_checkresponse}
  \EndFunction
  \end{algorithmic}
  \caption{An algorithm for evaluating the fairness of a prediction model $\widehat{Y}(\cdot)$ in a u-monotonic SCM $\sM$ with sensitive variables $\+S$ and response variables $\+Y$. The case to be considered is defined by conditions $\sW$ and $\sC$ where $\sW = (W_1=w_1) \land \cdots \land (W_M = w_M)$ denotes the conditions for $\Pa(\+Y) \setminus \+S$ (the non-sensitive observed parents of the responses $\+Y$), and $\sC = (C_1=c_1) \land \cdots \land (C_K = c_K)$ denotes the conditions for some other variables (which may include $\+S$). The argument $n$ defines the number of simulated counterfactual observations.}
  \label{alg:evaluate_fairness}
\end{algorithm}

Algorithm~\ref{alg:evaluate_fairness} loops over the possible values of the sensitive variables (line~\ref{line:loop_values}). If some sensitive variables are continuous, the values to be considered must be specified by the user. On line~\ref{line:call_simulatecounterfactual}, data are simulated for the intervention $\doo(\+S=\+s, \sW)$ that intervenes the sensitive variables but keeps the parents (outside of $\+S$) of the responses fixed to their original values. On line~\ref{line:call_predictionmodel}, the prediction model is applied to the simulated data. On line~\ref{line:call_checkresponse}, the fairness of the obtained predictions is evaluated, i.e., it is checked that all the predictions are the same regardless of the intervention applied to the sensitive variables. Algorithm~\ref{alg:evaluate_fairness} is usually called many times with different conditions $\mathcal{C}$. Algorithms~\ref{alg:continuous_condition}--\ref{alg:evaluate_fairness} are implemented in the R package \texttt{R6causal} 
\if0\blind{
\citepp{R6causal_anonymized}
} \fi
\if1\blind{
\citepp{R6causal}
} \fi
which contains R6 \citepp{R6} classes and methods for SCMs.

Algorithms for alternative counterfactual definitions of fairness could be implemented similarly. For instance, the definition given in Equation~\eqref{eq:kusnerfairness} emerges if the intervention $\doo(\+S=\+s,\sW)$ is replaced by the intervention $\doo(\+S=\+s)$ in Algorithm~\ref{alg:evaluate_fairness}.

\subsection{Conditional Simulation as a Particle Filter} \label{subsec:theoretical}

We show that Algorithm~\ref{alg:multiple_conditions} can be interpreted as a particle filter. This interpretation allows us to study the theoretical properties of the simulated sample. The notation used in this section is conventional for particle filters and differs from the rest of the paper. More specifically, we assume without loss of generality the topological order $V_1 < V_2 < \cdots < V_J$ which implies $\Pa(V_j) \subseteq \{V_1,\ldots,V_{j-1} \}$. We also assume that $\Er{j}$, $j=1,\ldots,J$ is the \errorterm{} of $V_j$. We introduce a shortcut notation $V_{1:j}=\{ V_1,\ldots,V_{j} \}$ and use $V_j^i$ to denote the $i^{\textrm{th}}$ observation of $V_j$. We will first review a basic particle filter algorithm and then show how it corresponds to Algorithm~\ref{alg:multiple_conditions}.

Particle filters are sequential Monte Carlo algorithms, which use sequentially defined `proposal distributions' consisting of the initial distribution $M_0(\ud z_0)$ on $\mathsf{Z}_0$ and transition probabilities $M_j(\ud z_j\mid z_{1:j-1})$ to $\mathsf{Z}_j$, and non-negative `potential functions' $G_j(z_{0:j})$, where $j$ is the time index \citep[cf.][]{del-moral}. 
Algorithm \ref{alg:pf_new} summarizes the basic particle filter algorithm,
\begin{algorithm}
  \caption{\textsc{ParticleFilter}$(M_{0:J}, G_{1:J}, n)$}
  \label{alg:pf_new} 
  \newcommand{\ind}[1]{{#1}}
\begin{algorithmic}[1]
\State
Draw $Z_0^{\ind{i}} \sim M_0(\uarg)$ and set $\vec{Z}_0^{\ind{i}} = Z_0^{\ind{i}}$ for $i=1{:}n$
\For{$j=1,\ldots,J$}
\State Draw $Z_{j}^{\ind{i}} \sim M_{j}(\uarg \mid
          \vec{Z}_{j-1}^{\ind{i}})$ for $i=1{:}n$
\State Calculate weights $W_j^i = \frac{\widetilde{W}_j^i }{ \sum_{j=1}^n \widetilde{W}_j^j}$ where 
$\widetilde{W}_j^i = G_{j}(\vec{Z}_{j-1}^{\ind{i}}, Z_j^{\ind{i}})$ for $i=1{:}n$ \label{step:calculate_weights}
\State  Draw $A_{j}^{\ind{1:n}} \sim \mathrm{Categorical}(
W_j^{1:n})$ and set 
          $\vec{Z}_{j}^{\ind{i}} 
          = (\vec{Z}_{j-1}^{\ind{A_j^{\ind{i}}}}, Z_{j}^{\ind{i}})$ for $i=1{:}n$
          \label{step:resample}
\EndFor
\State \textbf{output}
$\vec{Z}_J^{\ind{1:n}}$
\end{algorithmic}
\end{algorithm}
which outputs (approximate and dependent) samples from the probability distribution $\pi_J(B) = \gamma_J(B)/ \gamma_J(\mathsf{Z}_J)$, where $B \subseteq \mathsf{Z}_J$  and the unnormalised 'smoothing distribution' $\gamma_J$ is defined as follows:
\[
\gamma_J(B) = \int 
   M_0(\ud z_0) \prod_{j=1}^J M_j(\ud z_j \mid z_{1:j-1}) G_j(z_{0:j}) I(z_{0:J}\in B).
\]
For the reader's convenience, we state a mean square error convergence and central limit theorem for particle filter estimates \cite[Propositions 11.2 and 11.3]{chopin2020introduction}.
\begin{theorem}
  \label{thm:pf-conv_new}
Assume that the potential functions are bounded: $\| G_j \|_\infty = \sup_{\vec{z}} G_j(\vec{z}) < \infty$. Then, there exist constants $b,\sigma_h<\infty$ such that for any bounded test function $h:\mathsf{Z}_0\times\cdots\times\mathsf{Z}_J\to\R$ the output of Algorithm \ref{alg:pf_new} satisfies:
\begin{align*}
  \E\bigg[ \frac{1}{n} \sum_{i=1}^n h(\vec{Z}_J^i) - \pi_J(h) \bigg]^2 
  & \le b \frac{\|h\|_\infty^2}{n} , & \text{for all }n\ge 1,\\
  \frac{1}{\sqrt{n}} \sum_{i=1}^n \big[h(\vec{Z}_J^i) - \pi_J(h) \big]
  & \xrightarrow{n\to\infty} N(0,\sigma_{h}^2), & \text{in distribution,}
\end{align*}
where $\pi_J(h)$ stands for the expected value of $h(\+X)$ where $\+X$ follows distribution $\pi_J$.
\end{theorem}
Theorem \ref{thm:pf-conv_new} is stated for bounded test functions only. It is possible to prove similar results for unbounded functions \citep[e.g.,][and references therein]{cappe-moulines-ryden,rohrback-jack}.

In order to cast Algorithm~\ref{alg:multiple_conditions} as a particle filter, we need the following ingredients:

\begin{enumerate}[(a) ]
\item Initial particle states contain the global background variables $Z_0 = \widetilde{\vec{U}}$,
  \item Initial distribution $M_0(\ud  \widetilde{\vec{u}}) = p(\widetilde{\vec{u}}) \ud \widetilde{\vec{u}}$ on $\mathsf{Z}_0 = \mathrm{dom}(\widetilde{\+U})$, 
  \item The particle states for $j \geq 1$ contain the observed variables and their \errorterm{}s $Z_j = (V_j, \Er{j})$,
  \item If $V_j \in \+C$ then the value of $V_j$ is set to be $v_j$,
  \item Proposal distributions on $\mathsf{Z}_j = \R^2$:
  \begin{enumerate}[(i) ]
      \item If $V_j$ is continuous and $V_j \notin \+C$: 
\[    
    M_j(\ud v_j, \ud \er{j} \cond \widetilde{\vec{u}}, \overline{\+u}_{1:j-1}, \+v_{1:j-1}) = p(\er{j})\ud \er{j} \,I( g_{V_j}(\er{j};\Pau(v_j)) \in \ud v_j),
\]
    \item If $V_j$ is continuous and $V_j \in \+C$: 
    \[
    M_j(\ud v_j, \ud \er{j} \cond \widetilde{\vec{u}}, \overline{\+u}_{1:j-1}, \+v_{1:j-1}) = I( g_{V_j}^{-1}(v_j;\Pau(v_j)) \in \ud u_{j})I(v_j \in \ud v_j),
        \]
 \item If  $V_j$ is discrete:
\[    
    M_j(\ud v_j, \ud \er{j} \cond \widetilde{\vec{u}}, \overline{\+u}_{1:j-1}, \+v_{1:j-1}) = p(\er{j})\ud \er{j},
\]
  \end{enumerate} 
    \item Potentials: 
  \begin{enumerate}[(i) ]
    \item If $V_j \notin \+C$: $G_j \equiv 1$,
    \item If  $V_j \in \+C$ and $V_j$ is continuous: 
\[
G_j(\widetilde{\vec{u}}, \overline{\+u}_{1:j}, \+v_{1:j}) = 
    \begin{cases}
    \displaystyle
      \frac{\xi_{V_j}(u_{j})}{g_{V_j}'(u_{j}; \Pau(v_j))}, & \text{if }v_j \in \mathrm{dom}\big(g_{V_j}(u_{j}; \Pau(v_j))\big), \\
      0, & \text{otherwise}.
    \end{cases}
 \]    
 \item If $V_j \in \+C$ and $V_j$ is discrete:
 \[
G_j(\widetilde{\vec{u}}, \overline{\+u}_{1:j}, \+v_{1:j}) = 
    \begin{cases}
    \displaystyle
      1, & \text{if }I(v_j \in \ud v_j), \\
      0, & \text{otherwise}.
    \end{cases}
 \]    
 
  \end{enumerate}
\end{enumerate}

 We conclude that Algorithm~\ref{alg:multiple_conditions} is a particle filter algorithm and therefore Theorem~\ref{thm:pf-conv_new} is applicable:
\begin{corollary}
Let $\sM=(\+V,\+U,\+F,p(\+u))$ be a u-monotonic SCM and obtain a sample $\+D$ by calling $\+D = \Call{SimulateMultipleConditions}{n, \sM, \sC}$, where $\sC$ is a set of conditions on $\+C \subset \+V$. Further, let $h: \mathrm{dom}(\+V) \rightarrow \R$ be a bounded function and let $\pi_J$ be the conditional distribution of $\+V$ given $\sC$. If the conditional density of \mbox{$V_j \cond  \Pau(V_j)=\Pau(v_j)$} is bounded for all $V_j \in \+C$ then  the mean square error and central limit theorem as in Theorem \ref{thm:pf-conv_new} hold for $h(\+D)$ and $\pi_J$.
\end{corollary}
\begin{proof}
According to Theorem~\ref{thm:pf-conv_new}, the potential functions must be bounded. The potential functions for Algorithm~\ref{alg:continuous_condition} are   $\xi_{V_j}(u_{j})/{g_{V_j}'(u_{j}; \Pau(v_j))}$ which according to Lemma~\ref{lem:dedicated-transformation_new} give the conditional density of $V_j \cond  \Pau(V_j)=\Pau(v_j)$. 
\end{proof} 
  
We conclude the section with some remarks about possible extensions and elaborations of the algorithm. Algorithm~\ref{alg:pf_new} is similar to the seminal algorithm by \citett{gordon-salmond-smith}, which can be elaborated in a number of ways so that the algorithm remains (asymptotically) valid. The resampling indices $A_j^{1:n}$ in step~\ref{step:resample} need not be independent, as long as they form an unbiased `sample' from the categorical distribution \citep[e.g.,][and references therein]{douc2005comparison,chopin-singh-soto-vihola}. Furthermore, resampling can be (sometimes) omitted, and replaced by cumulation of the weights, and this can be done in an adaptive fashion \citepp{liu-chen-blind}. As a special case of adaptive resampling, we could omit the resampling on line~\ref{line:call_sample} of Algorithm~\ref{alg:continuous_condition} entirely and return a weighted sample instead. In Algorithm~\ref{alg:multiple_conditions}, the weights would be multiplied after processing each condition and the resampling would be carried out only as the final step. However, this approach may not work well with multiple conditions because, on line~\ref{line:call_simulatecontinuouscondition} of Algorithm~\ref{alg:multiple_conditions}, the current data $\+D[,\+S]$ would mostly contain observations with low weights. Resampling after each condition makes sure that new data are simulated only for relevant cases.  

The particle filter is known to be often efficient, if we are interested in the variables $Z_J$ only, or $Z_{J-\ell}$ for moderate lag $\ell$. However, as we care also about $Z_0$, large $J$ often causes the samples to contain only a few or one unique value for $Z_0$. This is a well-known problem, which is referred to as `sample degeneracy' \citepp{li2014fight}. Algorithm~\ref{alg:multiple_conditions} remains effective for an SCM with a high total number of variables, as long as the number of conditioning variables $J$ remains moderate.

\section{Simulations} \label{sec:simulations}

The critical part of Algorithm~\ref{alg:simulate_counterfactual} is the quality of the sample returned by Algorithm~\ref{alg:multiple_conditions}. In this section, the performance of Algorithm~\ref{alg:multiple_conditions} is studied using randomly generated linear Gaussian SCMs. Importantly, we can analytically derive the true distribution for comparison in this particular scenario, see Online Appendix 2 for the details.

In the simulation experiment, we generate linear Gaussian SCMs with random graph structure and random coefficients, apply Algorithm~\ref{alg:multiple_conditions}, and compare the simulated observations with the true conditional normal distribution.
The parameters of the simulation and the performance measures are explained in Online Appendix 3.

The key results are summarized in Table~\ref{tab:simulation_results}. The sample sizes in the simulation were $1000$, $10000$, $100000$, and $1000000$. For each setting, there were $1000$ simulation rounds. Case~A with five observed variables and one condition is expected to be an easy task. Cases B and C with ten observed variables represent moderate difficulty, and cases D and E with fifty observed variables are expected to be challenging. 

\begin{table}[!!ht]
\caption{\label{tab:simulation_results}
The results of the simulation experiment. The first panel shows the SCM parameters: the number of observed variables $|\+V|$, the number of conditions, the expected number of neighbors for the observed variables, and the expected number of global background variables per observed variable. The second panel reports the performance measures for cases A--E with different sample sizes $n$. The measures are the percentage of unique observations out of the total of $n$ observations (column `uniq. \%', ideally 100\%), the mean of sample means of standardized variables ($\overline{\overline{z}}$, ideally 0) and its minimum and maximum, the mean of sample standard deviations ($\overline{s}_z$, ideally 1) and its minimum and maximum, the average of Kolmogorov-Smirnov statistics (K-S, ideally 0), and the average difference between the true and the sample correlation coefficients (diff. cor., ideally 0).} 

\vspace{0.5cm}
\centering
\footnotesize
\begin{tabular}{lcccc}
\multicolumn{5}{c}{\textbf{Simulation cases and their parameters}} \\
Case & $|\+V|$ & conditions & mean neighbors & mean $|\widetilde{\+U}|/|\+V|$\\
 A  &  5  &  1  &  3  &  0 \\ 
 B  &  10  &  4  &  5  &  1 \\ 
 C  &  10  &  9  &  5  &  1 \\ 
 D  &  50  &  2  &  5  &  1 \\ 
 E  &  50  &  9  &  7  &  1 
\end{tabular}
\\[0.4cm]
\begin{tabular}{lrcrccccr}
\multicolumn{9}{c}{\textbf{Simulation results}} \\
Case & $n$ & uniq. \% & \multicolumn{1}{c}{$\overline{\overline{z}}$} & $\overline{z}$: min,max & $\overline{s}_z$ & $s_z$: min,max & K-S & diff. cor.  \\
  A  &  1000  &  49  &  $0.00$  &  ($-0.32$,\,$0.49$)  &  1.00  &  (0.00,\,1.13)  &  0.05  &  $0.00$ \\ 
 A  &  10000  &  49  &  $-0.00$  &  ($-0.84$,\,$0.09$)  &  1.00  &  (0.74,\,1.10)  &  0.02  &  $0.00$ \\ 
 A  &  100000  &  49  &  $0.00$  &  ($-0.02$,\,$0.04$)  &  1.00  &  (0.97,\,1.03)  &  0.00  &  $-0.00$ \\ 
 A  &  1000000  &  49  &  $0.00$  &  ($-0.02$,\,$0.01$)  &  1.00  &  (0.99,\,1.01)  &  0.00  &  $-0.00$ \\ 
 B  &  1000  &  17  &  $-0.01$  &  ($-3.01$,\,$2.13$)  &  0.95  &  (0.00,\,1.56)  &  0.17  &  $0.00$ \\ 
 B  &  10000  &  18  &  $-0.00$  &  ($-1.24$,\,$1.03$)  &  0.99  &  (0.02,\,1.48)  &  0.06  &  $0.00$ \\ 
 B  &  100000  &  17  &  $-0.00$  &  ($-0.47$,\,$0.25$)  &  1.00  &  (0.60,\,1.29)  &  0.02  &  $0.00$ \\ 
 B  &  1000000  &  16  &  $-0.00$  &  ($-0.17$,\,$0.09$)  &  1.00  &  (0.93,\,1.06)  &  0.01  &  $-0.00$ \\ 
 C  &  1000  &  18  &  $-0.00$  &  ($-1.58$,\,$1.11$)  &  0.97  &  (0.07,\,1.74)  &  0.16  &  --- \\ 
 C  &  10000  &  19  &  $0.00$  &  ($-1.19$,\,$0.86$)  &  1.00  &  (0.55,\,1.40)  &  0.06  &  --- \\ 
 C  &  100000  &  19  &  $0.00$  &  ($-0.19$,\,$0.15$)  &  1.00  &  (0.90,\,1.24)  &  0.02  &  --- \\ 
 C  &  1000000  &  20  &  $0.00$  &  ($-0.04$,\,$0.09$)  &  1.00  &  (0.88,\,1.04)  &  0.01  &  --- \\ 
 D  &  1000  &  14  &  $-0.02$  &  ($-2.41$,\,$2.14$)  &  0.92  &  (0.00,\,1.63)  &  0.19  &  $-0.02$ \\ 
 D  &  10000  &  13  &  $0.00$  &  ($-1.01$,\,$1.53$)  &  0.98  &  (0.00,\,1.67)  &  0.07  &  $-0.00$ \\ 
 D  &  100000  &  14  &  $-0.00$  &  ($-0.99$,\,$0.59$)  &  1.00  &  (0.33,\,1.27)  &  0.02  &  $-0.00$ \\ 
 D  &  1000000  &  13  &  $-0.00$  &  ($-0.22$,\,$0.07$)  &  1.00  &  (0.89,\,1.08)  &  0.01  &  $0.00$ \\ 
 E  &  1000  &  4  &  $0.02$  &  ($-4.04$,\,$4.05$)  &  0.38  &  (0.00,\,1.66)  &  0.60  &  $-0.05$ \\ 
 E  &  10000  &  4  &  $-0.01$  &  ($-3.63$,\,$2.96$)  &  0.67  &  (0.00,\,1.97)  &  0.41  &  $-0.01$ \\ 
 E  &  100000  &  4  &  $0.01$  &  ($-2.98$,\,$2.74$)  &  0.84  &  (0.00,\,2.21)  &  0.25  &  $0.01$ \\ 
 E  &  1000000  &  4  &  $-0.02$  &  ($-1.76$,\,$2.10$)  &  0.95  &  (0.00,\,1.70)  &  0.12  &  $0.01$ 
\end{tabular}
\end{table}
\clearpage

It can be seen that the percentage of unique observations decreases when the setting becomes more challenging. The repeated observations originate from the resampling on line~\ref{line:call_sample} of Algorithm~\ref{alg:continuous_condition}.  In all cases, $\overline{\overline{z}}$, the mean of sample means was unbiased, and the variation of sample means between simulation rounds decreased as the sample size increased. The mean of sample standard deviations was close to the true value $1$ in all cases except in case E where the variation was too small in average. This was reflected also in the Kolmogorov--Smirnov statistics that indicated major differences from the true distribution. 
The correlation coefficients were unbiased or almost unbiased in all cases. The average runtime per a simulation round in case E with $n=1000000$ was 121 seconds. The simulation code is available at 
\url{https://github.com/JuhaKarvanen/simulating_counterfactuals}.

\section{Application to Fairness in Credit-Scoring} \label{sec:fairness}

In this section, we show how Algorithm~\ref{alg:evaluate_fairness} can be used in the fairness analysis of a synthetic scenario where the details of the prediction model are unknown and real data are not available. We consider credit-scoring where the decision to grant a bank loan to a consumer depends on personal data and the amount of the loan. The credit decision is made by an automated prediction model that can be accessed via an application programming interface (API) but is otherwise unknown to the fairness evaluator. The outcome of the prediction model is the default risk. We consider three models:
\begin{enumerate}[A) ]
\item The prediction model has no restrictions for the predictors.
\item The prediction model does not directly use sensitive variables.
\item The prediction model uses only non-sensitive variables that have a direct causal effect on the outcome.
\end{enumerate}
It is expected that the evaluation should indicate that only prediction model C is fair according to Definition~\ref{alg:evaluate_fairness}.

For the fairness evaluation, we constructed an SCM with the causal diagram shown in Figure~\ref{fig:credit}. We consider variables that are similar to those typically present in credit-scoring datasets, such as Statlog (German Credit Data) \citepp{misc_statlog_(german_credit_data)_144}. Here, \textit{default} is the outcome variable that indicates whether the customer will repay the loan or not. The annual income (in euros, continuous), the savings (in euros, continuous), the credit amount (in euros, continuous), type of housing (categorical), level of education (ordinal), the type of job (categorical), the length of employment (in months, discrete) and ethnicity (categorical) have a direct effect on the probability of default. The marital status (categorical), the number of children (discrete), age (in years, continuous), and gender (categorical) have only an indirect effect on the risk of default. Ethnicity and gender are the sensitive variables. The address does not have a causal effect on the risk of default but is included here as a potential proxy of ethnicity. The causal diagram also has five unobserved confounders $U_1,\ldots,U_5$. The \errorterm{}s are not displayed. The detailed structure of the SCM is explained in Online Appendix~4 and the R code for the example is available in the repository \url{https://github.com/JuhaKarvanen/simulating_counterfactuals} in the file \texttt{fairness\_example.R}.

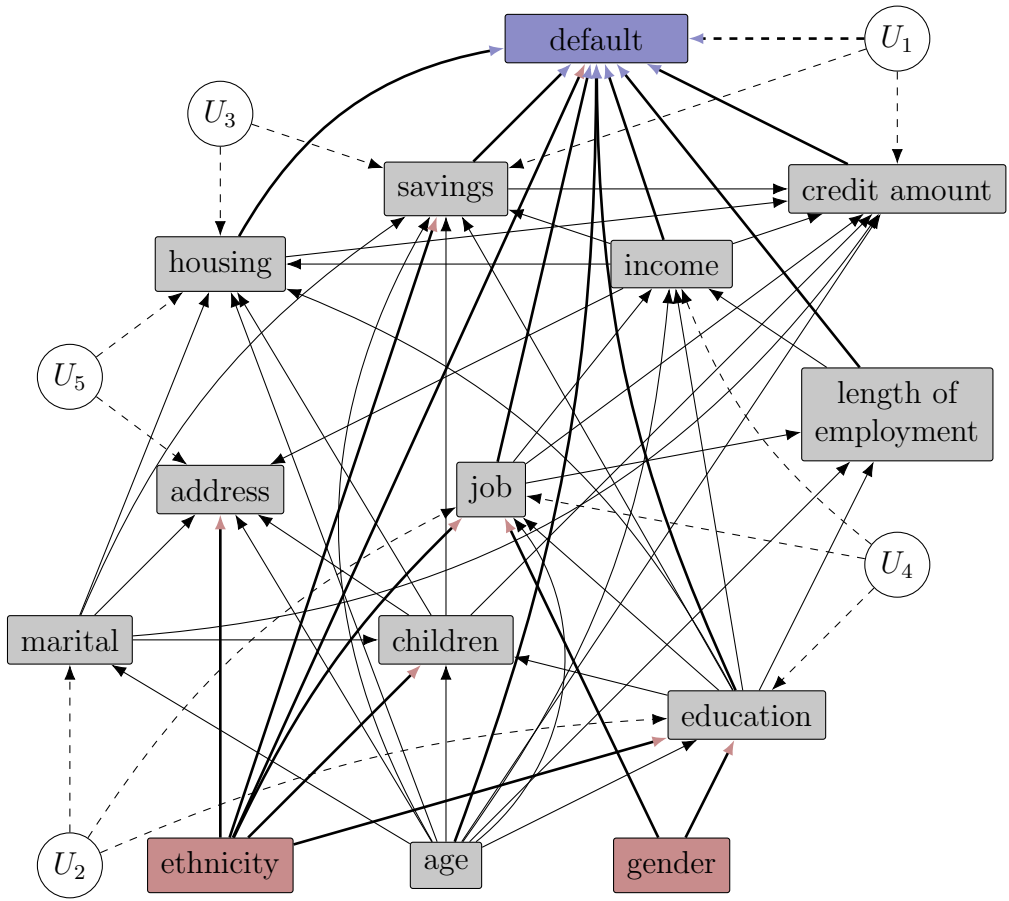
\begin{figure}
\begin{center}
\begin{tikzpicture}[xscale=1,yscale=1] 
  \node [obsout = {default}] at (7,11) {\hphantom{00}default\hphantom{00}};
  \node [obs = {credit_amount}] at (11,9) {credit amount};
  \node [obs = {length_of_employment}, align=center] at (11,6) {length of \\ employment};
  \node [obs = {savings}] at (5,9) {savings};
  \node [obs = {address}] at (2,5) {address};
  \node [obs = {housing}] at (2,8) {housing};
  \node [obs = {income}] at (8,8) {income};
  \node [obs = {job}] at (5.6,5) {job};
  \node [obs = {children}] at (5,3) {children};
  \node [obs = {education}] at (9,2) {education};
  \node [obssens = {gender}] at (8,0) {gender};
  \node [obs = {marital}] at (0,3) {marital};
  \node [obssens = {ethnicity}] at (2,0) {ethnicity};
  \node [obs = {age}] at (5,0) {age};
  \node [lat = {u5}] at (0,6.5) {$U_5 $};
  \node [lat = {u4}] at (11,4) {$ U_4 $};
  \node [lat = {u3}] at (2,10) {$ U_3 $};
  \node [lat = {u2}] at (0,0) {$ U_2 $};
  \node [lat = {u1}] at (11,11) {$ U_1 $};
  \path [->, dashed] (u1) edge (savings);
  \path [->, dashed] (u1) edge (credit_amount);
  \path [->, dashed, draw=outcomecolor, fill=outcomecolor, line width=1pt] (u1) edge (default); 
  \path [->, dashed] (u2) edge (marital);
  \path [->, dashed] (u2) edge [bend left=10] (education);
  \path [->, dashed, bend left=15] (u2) edge (job);
  \path [->, dashed] (u3) edge (housing);
  \path [->, dashed] (u3) edge (savings);
  \path [->, dashed] (u4) edge (education);
  \path [->, dashed] (u4) edge (job);
  \path [->, dashed] (u4) edge [bend left=15] (income);
  \path [->, dashed] (u5) edge (housing);
  \path [->, dashed] (u5) edge (address);
  \path [->] (age) edge (marital);
  \path [->] (age) edge (education);
  \path [->] (age) edge (children);
  \path [->] (age) edge [bend right=45] (job);
  \path [->] (age) edge [bend right=15] (income);
  \path [->] (age) edge (housing);
  \path [->] (age) edge (address);
  \path [->] (age) edge [bend left] (savings);
  \path [->] (age) edge (length_of_employment); 
  \path [->] (age) edge (credit_amount); 
  \path [->, draw=outcomecolor, fill=outcomecolor, line width=1pt] (age) edge [bend right=10] (default); 
  \path [->, draw=sensitivecolor, fill=sensitivecolor, line width=1pt] (ethnicity) edge (education); 
  \path [->, draw=sensitivecolor, fill=sensitivecolor, line width=1pt] (ethnicity) edge (children); 
  \path [->, draw=sensitivecolor, fill=sensitivecolor, line width=1pt, bend left=10] (ethnicity) edge (job); 
  \path [->, draw=sensitivecolor, fill=sensitivecolor, line width=1pt] (ethnicity) edge (address); 
  \path [->, draw=sensitivecolor, fill=sensitivecolor, line width=1pt] (ethnicity) edge (savings); 
   \path [->, draw=sensitivecolor, fill=sensitivecolor, line width=1pt, bend left = 0] (ethnicity) edge (default);
  \path [->] (marital) edge (children);
  \path [->] (marital) edge (housing);
  \path [->] (marital) edge (address);
  \path [->] (marital) edge [bend left=16] (savings);
  \path [->] (marital) edge [bend right=25] (credit_amount);
  \path [->, draw=sensitivecolor, fill=sensitivecolor, line width=1pt] (gender) edge (education);
  \path [->, draw=sensitivecolor, fill=sensitivecolor, line width=1pt] (gender) edge (job);
  \path [->] (education) edge (children);
  \path [->] (education) edge (job);
  \path [->] (education) edge (income);
  \path [->] (education) edge [bend right=20] (housing);
  \path [->] (education) edge (savings);
  \path [->] (education) edge (length_of_employment); 
  \path [->, draw=outcomecolor, fill=outcomecolor, line width=1pt] (education) edge [bend left=12] (default);
  \path [->] (children) edge (housing);
  \path [->] (children) edge (address);
  \path [->] (children) edge (savings);
  \path [->] (children) edge (credit_amount);
  \path [->] (job) edge (income); 
  \path [->] (job) edge (length_of_employment); 
  \path [->] (job) edge (credit_amount); 
  \path [->, draw=outcomecolor, fill=outcomecolor, line width=1pt] (job) edge (default); 
  \path [->] (income) edge (housing); 
  \path [->] (income) edge (address); 
  \path [->] (income) edge (savings); 
  \path [->] (income) edge (credit_amount); 
  \path [->, draw=outcomecolor, fill=outcomecolor, line width=1pt] (income) edge (default); 
  \path [->] (housing) edge (credit_amount); 
  \path [->, draw=outcomecolor, fill=outcomecolor, line width=1pt] (housing) edge [bend left=25] (default);
  \path [->] (savings) edge (credit_amount); 
  \path [->, draw=outcomecolor, fill=outcomecolor, line width=1pt] (savings) edge (default); 
  \path [->] (length_of_employment) edge (income); 
  \path [->, draw=outcomecolor, fill=outcomecolor, line width=1pt] (length_of_employment) edge (default); 
  \path [->, draw=outcomecolor, fill=outcomecolor, line width=1pt] (credit_amount) edge (default); 
\end{tikzpicture}
\caption{Causal model for the fairness of credit-scoring example. Ethnicity and gender (red nodes) are the sensitive variables, and the risk of default (blue node) is the outcome to be predicted. Gray nodes depict other observed variables and white nodes are unobserved confounders.} \label{fig:credit}
\end{center}
\end{figure}

To set up the example, we simulated training data from an SCM corresponding to the causal diagram of Figure~\ref{fig:credit} and fitted prediction models A, B, and C for the default risk using XGBoost \citepp{xgboost,xgboostpackage}. These models are opaque AI models for the fairness evaluator who can only see the probability of default predicted by the models. 

Algorithm~\ref{alg:evaluate_fairness} was applied to prediction models A, B, and C in 1000 cases that were again simulated from the same SCM. In the algorithm, $\widehat{Y}(\cdot)$ was one of the prediction models, $\mathcal{M}$ was the SCM whose causal diagram is depicted in Figure~\ref{fig:credit}, sensitive variables $\+S$ were gender and ethnicity, condition $\mathcal{C}$ contained all observed values of the case, and number of observations was $n=1000$. For each case and each model, the counterfactual probability of default was estimated for all possible combinations of gender and ethnicity. For a fair prediction model, these probabilities should be the same regardless of gender and ethnicity as stated in Definition~\ref{def:fairness}. Consequently, the difference between the largest and the smallest probability was chosen as a measure of fairness. For instance, if the estimated counterfactual probability of default was $0.09$ for one combination of gender and ethnicity and $0.08$ for all other combinations, the counterfactual difference would be $0.09-0.08=0.01$.

The fairness results based on the 1000 cases are presented in Table~\ref{tab:fairness_results}. Only prediction model C was evaluated to be fair. Even if the median counterfactual difference was close to zero for models A and B, there were cases where the difference was very large. Consequently, the use of models A and B in decision-making would potentially lead to illegal discrimination.

\begin{table}[h]
\caption{The fairness results for the credit-scoring models. The first row indicates the percentage of cases where the counterfactual difference was exactly 0, i.e., the fairness criterion held exactly. The second row shows the percentage of cases where the difference was smaller than $0.01$, i.e., the counterfactual probabilities were not necessarily the same but close to each other. The third and fourth rows report the median and maximum of the counterfactual difference, respectively. 
} \label{tab:fairness_results}
\begin{center}
\begin{tabular}{lrrr}
& \multicolumn{3}{c}{Model} \\
Measure & A & B & C\\
\hline
 Zero difference (\%)   &  18  &  27  &  100 \\ 
 Difference $<0.01$ (\%)   &  90  &  94  &  100 \\ 
 Median difference  &  0.00031  &  0.00032  &  0.00000 \\ 
 Maximum difference  &  0.57  &  0.25  &  0.00 
\end{tabular}
\end{center}
\end{table}

It would have been interesting to compare the results of Algorithm~\ref{alg:evaluate_fairness} with other methods of fairness evaluation but this was unattainable. Earlier works on counterfactual fairness \citepp{kusner2017counterfactual,nabi2018fair,chiappa2019path,wu2019pc-fairness,richens2022counterfactual} give some examples but do not provide a code that a user could directly apply to a new problem. On the other hand, \texttt{Fairlearn} \citepp{weerts2023fairlearn}  is a general purpose package for fairness evaluation but it does not currently include counterfactual definitions or metrics.

\section{Discussion} \label{sec:discussion}

We presented an algorithm for simulating data from a specified counterfactual distribution (Algorithm~\ref{alg:simulate_counterfactual}). The algorithm receives a known SCM and a counterfactual of interest  as inputs and simulates independent but non-unique observations from the counterfactual distribution. The algorithm is intended to be a tool for simulation studies on counterfactual inference and fairness evaluation. In Section~\ref{sec:fairness}, we demonstrate how the algorithm can be used in fairness analysis when the prediction models are unknown and real data are not available. 
 
The proposed algorithm can be interpreted as a particle filter and it returns an approximate sample that yields asymptotically valid inference. The algorithm possesses multiple practical strengths. It is applicable to SCMs that may simultaneously have both continuous and discrete variables and may also have unobserved confounders. The algorithm solves the non-trivial problem of conditioning on continuous variables and operates rapidly with SCMs of sizes typically found in the literature. A documented open-source software implementation is available.

The counterfactual simulation algorithm has some requirements and limitations. The knowledge of the full causal model is a strong assumption in practical applications. However, counterfactual distributions cannot in general be simulated, for example, based on data alone, unless the counterfactual distribution in question is identifiable from the available data sources \citepp{shpitser2007counterfactuals,tikka2022identifying}. Our approach could also be combined with a sensitivity analysis where the counterfactual simulation is repeated for different model parameters.

We presented the algorithm for SCMs which are u-monotonic with respect to continuous conditioning variables.
We do not consider this as a serious restriction because the monotonicity requirement concerns only a subset of variables and offers flexibility for the functional form of the \errorterm{}s' impact. Similar but more restrictive monotonicity assumptions have been also used for causal normalizing flows \citepp{javaloy2023causal}. 
In principle, the algorithms could be generalized to work with \errorterm{}s that are \emph{piecewise} monotonic. Then, the expression in Lemma~\ref{lem:dedicated-transformation_new} involves a sum of domain-restricted terms, and the dedicated error terms are no longer uniquely determined, and need to be simulated.

We note that our method can be generalized to multivariate observations (clustered variables) \citepp{tikka2021clustering} with multivariate dedicated error terms. Namely, if the multivariate transformation is a differentiable bijection, we may modify Lemma~\ref{lem:dedicated-transformation_new} using a multivariate transformation formula. However, in this case, the inverse transformation and the Jacobian must be available for point-wise evaluations.

Sample degeneracy is a well-known problem of particle filters that also affects the counterfactual simulation algorithm when the number of conditioning variables increases. The literature on particle filters proposes ways to address the degeneracy problem. One possibility is to run independent particle filters and weigh their outputs, which leads to consistent estimates \citep[cf.][Proposition 23]{vihola-helske-franks-preprint}. This is appealing because of its direct parallelizability but leads to a weighted sample. Another, perhaps even more promising approach, is to iterate the \emph{conditional} particle filter (CPF) \citepp{andrieu-doucet-holenstein}, which is a relatively easy modification of the particle filter algorithm, and which defines a valid Markov chain targeting the desired conditional distribution.

When carefully configured, the CPF can work with other resampling algorithms \citepp{chopin-singh,karppinen-singh-vihola}. CPFs with adaptive resampling have been suggested as well \citepp{lee-phd}. We leave the practical efficiency of the CPF and its variants for future research.

The significance of counterfactual simulation emerges from the context of fairness evaluation. The fairness evaluation algorithm (Algorithm~\ref{alg:evaluate_fairness}) uses  simulated data, which extends the scope of evaluation to encompass data that could be potentially encountered \citepp{cheng2021socially}. We perceive the proposed fairness evaluation algorithm as complementary to methods that are based on real data.  It is important to note that in addition to Definition~\ref{def:fairness}, the simulation algorithm can also be applied to other definitions of counterfactual fairness. Other potential applications of the counterfactual simulation algorithm include counterfactual explanations and algorithmic recourse \citepp{karimi2021algorithmic}. Furthermore, the conditional simulation (Algorithm~\ref{alg:multiple_conditions}) is applicable to Bayesian networks that lack a causal interpretation.

\section*{Acknowledgements}
CSC -- IT Center for Science, Finland, is acknowledged for computational resources. MV was supported by Research Council of Finland (Finnish Centre of Excellence in Randomness and Structures, grant 346311). ST was supported by Research Council of Finland (PREDLIFE: Towards well-informed decisions: Predicting long-term effects of policy reforms on life trajectories, grant 331817).

\bibliographystyle{apalike}
\bibliography{references}

\section*{Online Appendix 1: Simple Illustration of Counterfactual Inference}
Consider the following SCM $\sM$:
\begin{align*}
U_Z &\sim \N(0,1), \\
U_X &\sim \N(0,1), \\
U_Y &\sim \N(0,1), \\
Z &= U_Z,\\
X &= Z + U_X, \\
Y &= X + Z + U_Y,
\end{align*}
where $\+ U= \{U_Z, U_X, U_Y\}$ are unobserved background variables and $Z$, $X$ and $Y$ are observed variables.

The aim is to derive the counterfactual distribution
\[
p(Y_{\doo(X = -1)} =  y \cond Y = 1).
\]
The condition $Y = 1$ is fulfilled if $U_Y + U_X + 2 U_Z = 1$. This equation defines a plane in the space of $(U_Z, U_X, U_Y)$. In the general case, it might not be possible to present the surface of interest in a closed form but for the normal distribution this is doable.

Following \citepp{pearl:book2009}, the solution can be obtained in three steps:
\begin{enumerate}
\item Find the distribution $p(U_Z, U_X, U_Y \cond Y = 1)$.
\item Modify $\sM$ by the intervention $\doo(X = -1)$ to obtain the submodel $\sM_{\doo(X = -1)}$.
\item Use the submodel $\sM_{\doo(X = -1)}$ and the updated distribution $p(U_Z, U_X, U_Y \cond Y = 1)$ to compute the probability distribution of $Y_{\doo(X = -1)}$.
\end{enumerate}

The unconditional joint distribution of $(U_Z, U_Y, Y)$ is a multivariate normal distribution with zero expectations and the covariance matrix
\[
\begin{pmatrix}
 1 & 0 & 2 \\
 0 & 1 & 1\\
 2 & 1 & 6 \\
\end{pmatrix}.
\]
Because $X$ is determined by the intervention $\doo(X = -1)$ and $U_X$ affects only $X$, it suffices to consider only background variables $U_Z$ and $U_Y$.
The distribution of $(U_Z, U_Y)$ conditional on $Y = 1$ is a multivariate normal distribution with the expectation
\[
\begin{pmatrix}
 2 \\
 1 \\
\end{pmatrix} \frac{1}{6} = 
\begin{pmatrix}
 \frac{1}{3}  \\
 \frac{1}{6} \\
\end{pmatrix}
\]
and the covariance matrix
\[
\begin{pmatrix}
 1 & 0 \\
 0 & 1 \\
\end{pmatrix} -
\begin{pmatrix}
 2  \\
 1 \\
\end{pmatrix} \frac{1}{6}
\begin{pmatrix}
 2  & 1 \\
\end{pmatrix} =
\begin{pmatrix}
 \frac{1}{3} & -\frac{1}{3} \\
 -\frac{1}{3} & \frac{5}{6} \\
\end{pmatrix}.
\]

The submodel $\sM_{\doo(X = -1)}$ is the following: 
\begin{align*}
U_Z &\sim \N(0,1), \\
U_X &\sim \N(0,1), \\
U_Y &\sim \N(0,1), \\
Z &= U_Z,\\
X &= -1, \\
Y &= X + Z + U_Y = -1 + U_Z + U_Y.
\end{align*}
The counterfactual distribution of interest, i.e., the distribution of $Y=-1 + U_Z + U_Y$ on the condition of $Y=1$ is a normal distribution with the expectation 
\[
-1 + \frac{1}{3} + \frac{1}{6}  = -\frac{1}{2},
\]
and the variance 
\[
\frac{1}{3} + \frac{5}{6} - 2 \cdot \frac{1}{3} = \frac{1}{2}.
\]
In other words, $Y_{\doo(X = -1)} \cond (Y = 1) \sim \N(-\frac{1}{2},\frac{1}{2})$.

\section*{Online Appendix 2: Counterfactual Inference for Linear Gaussian SCMs}

Consider a linear Gaussian SCM with observed variables $\+V = (V_1,\ldots,V_J)$ and mutually independent background variables $\+U = (U_1,\ldots,U_H)$ that follow the standard normal distribution. The model is written as
\[
\+V = \+b_0 + \+B_1 \+V + \+B_2 \+U,
\]
where $\+b_0$ is a constant vector, $\+B_1$ is a $J \times J$ strictly triangular matrix, and $\+B_2$ is a $H \times H$ matrix. The observed variables can be expressed in terms of the background variables as follows
\[
\+V = (\+I - \+B_1)^{-1}(\+b_0  + \+B_2 \+U),
\]
and because $\+b_0$ is a constant and $\+U \sim \N(\+0,\+I)$ we have $\+V \sim \N(\boldsymbol{\mu}_\+V,\boldsymbol{\Sigma}_{\+V \+V})$ where 
\begin{align*}
 \boldsymbol{\mu}_\+V &= (\+I - \+B_1)^{-1}\+c, \\
 \boldsymbol{\Sigma}_{\+V \+V} &= (\+I - \+B_1)^{-1}(\+B_2 \+B_2^T) ((\+I - \+B_1)^{-1})^T.
\end{align*}

The joint distribution of $\+V$ and $\+U$ is
\[
\begin{pmatrix}
\+V \\
\+U
\end{pmatrix}
\sim
\N \left( 
\begin{pmatrix}
\boldsymbol{\mu}_\+V \\
\+0
\end{pmatrix},
\begin{pmatrix}
\boldsymbol{\Sigma}_{\+V \+V} & \boldsymbol{\Sigma}_{\+V \+U} \\
\boldsymbol{\Sigma}_{\+V \+U}^T & \+I
\end{pmatrix}
\right),
\]
where $\Sigma_{\+V \+U}=(\+I - \+B_1)^{-1}\+B_2$.

Next, we consider the conditional distribution when the values of some observed variables are fixed. We partition the observed variables as $\+V = \+V_1 \cup \+V_2$ such that the values of $\+V_2$ are fixed and write
\[
\begin{pmatrix}
\+V_1 \\
\+V_2 \\
\+U
\end{pmatrix}
\sim
\N \left( 
\begin{pmatrix}
\boldsymbol{\mu}_{\+V_1} \\
\boldsymbol{\mu}_{\+V_2} \\
\+0
\end{pmatrix},
\begin{pmatrix}
\boldsymbol{\Sigma}_{\+V_1 \+V_1} & \boldsymbol{\Sigma}_{\+V_1 \+V_2} & \boldsymbol{\Sigma}_{\+V_1 \+U} \\
\boldsymbol{\Sigma}_{\+V_2 \+V_1} & \boldsymbol{\Sigma}_{\+V_2 \+V_2} & \boldsymbol{\Sigma}_{\+V_2 \+U} \\
\boldsymbol{\Sigma}_{\+V_1 \+U}^T & \boldsymbol{\Sigma}_{\+V_2 \+U}^T & \+I
\end{pmatrix}
\right).
\]
The distribution of $(\+V_1, \+U)^T$ conditional on $\+V_2 = \+c$ is a normal distribution with the mean vector
\begin{equation} \label{eq:conditionalnormalmean}
\begin{pmatrix}
\boldsymbol{\mu}_{\+V_1} \\
\+0
\end{pmatrix} +
\begin{pmatrix}
\boldsymbol{\Sigma}_{\+V_1 \+V_2} \\
\boldsymbol{\Sigma}_{\+V_2 \+U}^T
\end{pmatrix}
\boldsymbol{\Sigma}_{\+V_2 \+V_2}^{-1}  (\+c - \boldsymbol{\mu}_{\+V_2} ),
\end{equation}
and the covariance matrix
\begin{equation} \label{eq:conditionalnormalcov}
\begin{pmatrix}
\boldsymbol{\Sigma}_{\+V_1 \+V_1}  & \boldsymbol{\Sigma}_{\+V_1 \+U} \\
\boldsymbol{\Sigma}_{\+V_1 \+U}^T  & \+I
\end{pmatrix} -
\begin{pmatrix}
\boldsymbol{\Sigma}_{\+V_1 \+V_2} \\
\boldsymbol{\Sigma}_{\+V_2 \+U}^T
\end{pmatrix}
\boldsymbol{\Sigma}_{\+V_2 \+V_2}^{-1}
\begin{pmatrix}
\boldsymbol{\Sigma}_{\+V_1 \+V_2}^T &
\boldsymbol{\Sigma}_{\+V_2 \+U}
\end{pmatrix}
.
\end{equation}

Finally, we derive the distribution after the intervention. The observed variables after the intervention can be expressed as
\[
\+V^\circ = (\+I - \+B_1^\circ)^{-1}(\+b_0^\circ  + \+B_2^\circ (\+U \cond \+V_2 =  \+c) ),
\]
where $\+b_0^\circ$ is obtained from $\+b_0$ by setting the constants of intervened variables to the value of the intervention, $\+B_1^\circ$ and $\+B_2^\circ$ are obtained from $\+B_1$ and $\+B_2$, respectively,  by setting the values in the row vectors of the intervened variables to zero, and the distribution of $\+U \cond \+V_2 = \+c$ is defined by equations~\eqref{eq:conditionalnormalmean} and \eqref{eq:conditionalnormalcov}.

\section*{Online Appendix 3: Details of the Simulation Experiment}

Here we describe additional details of the simulation experiment presented in Section~4 of the main paper . The parameters of the simulation experiment can be divided into the SCM parameters, the counterfactual parameters, and the parameters of Algorithm~3. 
The SCM parameters include the number of observed variables, the average number of neighbouring observed variables per observed variable, the average number of unobserved confounders per observed variable, and the probability distributions from where the coefficients in $\+B_1$, $\+B_2$, and $\+b_0$ (defined in Online Appendix 2) will be drawn. The counterfactual parameters include the number of conditioning variables. The only free parameter of  Algorithm~3 
is the sample size $n$. 

The performance measures of the simulation experiment evaluate the proportion of unique observations in the sample, univariate statistics of observed variables, and the covariance structure between variables. The proportion of unique observations is calculated by dividing the number of unique observations by the sample size. 

The univariate statistics are calculated for an arbitrarily chosen observed variable. The chosen variable is standardized by subtracting the true mean and dividing by the true standard deviation. 
The mean and standard deviation of the standardized variable $z$ should be $0$ and $1$, respectively. In addition, the Kolmogorov-Smirnov statistic measuring the largest difference between the true cumulative distribution function and the empirical cumulative distribution function is reported. The correlation coefficient between two arbitrarily chosen observed variables was compared to the true correlation coefficient.

\section*{Online Appendix 4: Details of the Credit-Scoring Causal Model}

The SCM depicted in Figure~1 of the main paper is explained here variable by variable in a topological order starting from the bottom of the graph. Ethnicity, age and gender do not have parents in the graph and are assumed to be independent from each other. Ethnicity is a categorical variable with with probabilities $0.75$, $0.15$ and $0.10$ for the classes $1$, $2$ and $3$, respectively. Age is uniformly distributed between $18$ and $78$ years. Gender has two values, $0$ and $1$, with equal probabilities. We have deliberately chosen not to label ethnicity and gender classes.

Education is an ordered factor with four levels ($1$ primary, $2$ secondary, $3$ tertiary, $4$ doctorate) and is affected by ethnicity, age, gender, and the unobserved confounders $U_2$ and $U_4$. On average, education is higher for ethnicity class $1$, for gender $0$, and for older individuals. Marital status has three classes $1$ (single), $2$ (married or cohabiting) and $3$ (divorced or widowed) and is affected by age and unobserved confounder $U_2$. The odds of class $2$ compared to class $1$ and the odds of class $3$ compared to class $2$ increase as a function of age. The number of children follows a Poisson distribution whose expectation is affected by ethnicity, age, marital status and education. On average, the number of children is higher for ethnicity classes $2$ and $3$, higher age (up to age $45$), marital status $2$ and $3$ and higher education.

Job is a categorical variable with three classes ($1$ not working and not retired, $2$ working,  $3$ retired) and is affected by ethnicity, age, gender, the number of children, education, and the unobserved confounders $U_2$ and $U_4$. Ethnicity classes 2 and 3, gender $0$, young age and low education increase the probability that a person to the job class $1$. The same features increase the odds of job class $2$ compared to job class $3$. The length of employment (in years) is a continuous variable that is affected by age and education that also determine a technical upper limit for the length of employment. Income is a continuous variable that describes the annual income (in euros) and is affected by job, education, age, the length of employment and unobserved confounder $U_4$. On average, income is higher for working individuals (job class $2$), higher education classes and longer length of employment. For job class $2$, income has its peak at age $58$.

Address is an ordered factor with classes $1,2,\ldots,10$ and is affected by marital status, ethnicity, age, the number of children, income and the unobserved confounder $U_5$. The address variable is thought to be derived from the street address in such a way that higher classes of address are associated with higher income. In addition, ethnicity classes $2$ and $3$ have higher odds for living in an address of class $1$ compared to ethnicity class $1$. Higher age, higher number of children and marital status $2$ and $3$ are associated with higher classes of address.
Housing has two possible values depending whether the home is rented (value $1$) or owned (value $2$). Housing is affected by age, marital status, the number of children, education, income and the unobserved confounders $U_3$ and $U_5$. Higher age, marital status $2$ and $3$, higher number of children, higher education and higher income increase the odds of home ownership.

Savings (in euros) is a continuous variable that is affected by income, age, ethnicity, education, marital status, and the unobserved confounders $U_1$ and $U_3$. Age and income have a joint effect on savings and it is assumed that on average $5$\% of income has been saved every year starting from age $18$. In addition, there is an age-specific reduction to savings that has its peak at age $27$. Higher education, ethnicity class $3$ and marital status $2$ (married or cohabiting) increase the amount of savings. The number of children decrease the amount of savings. Finally, there is a Gamma-distributed multiplier for the savings that depends on the background variables. This multiplier reflects the success of investments and inherited property.

Credit amount (in euros) is a continuous variable that is affected by age, income, job, housing, marital status, the number children, savings and the unobserved confounder $U_1$. The credit amount is the highest at age $40$ and increases as a function of income. On average, the credit amount is higher for individuals who are working, have rented their house, have marital status $2$ or $3$ and have a high number of children. The credit amount decreases as a function of savings. The minimum of credit amount is $5000$ euros.

Default is a binary variable that is affected by ethnicity, age, education, job, the length of employment, income, housing, savings, credit amount and the unobserved confounder $U_1$. Being a member of minority group (ethnicity classes $2$ and $3$) reduces the risk of default. Higher age, higher education, having a job, longer length of employment, higher income, home ownership and high amount of savings also reduce the risk of default. The risk of default increases as a function of credit amount.

The R code for the credit-scoring example is available in the repository \url{https://github.com/JuhaKarvanen/simulating_counterfactuals} in the file \texttt{fairness\_example.R}.

\end{document}